\documentclass{article}
\usepackage{microtype}
\usepackage{graphicx}
\usepackage{subfigure}
\usepackage{hyperref}
\usepackage{amsmath}
\usepackage{amsthm}
\usepackage{amssymb}
\usepackage{hyperref}
\usepackage{bbm}
\newtheorem{theorem}{Theorem}

\newtheorem{property}{Property}
\newtheorem{corollary}{Corollary}
\newtheorem{lemma}{Lemma}
\newtheorem*{lemma*}{Lemma}

\renewcommand{\implies}{\Rightarrow}
\usepackage[accepted]{icml2018}
\begin{document}
\twocolumn[
\icmltitle{On the Analysis of Trajectories of Gradient Descent in the Optimization of Deep Neural Networks}

\begin{icmlauthorlist}
\icmlauthor{Adepu Ravi Sankar}{aff1}
\icmlauthor{Vishwak Srinivasan}{aff1}
\icmlauthor{Vineeth N Balasubramanian}{aff1}
\end{icmlauthorlist}

\icmlaffiliation{aff1}{Department of Computer Science and Engineering, Indian Institute of Technology Hyderabad, Kandi, Telangana, India}

\icmlcorrespondingauthor{Adepu Ravi Sankar}{cs14resch11001@iith.ac.in}
\icmlcorrespondingauthor{Vishwak Srinivasan}{cs15btech11043@iith.ac.in}
\icmlcorrespondingauthor{Vineeth N Balasubramanian}{vineethnb@iith.ac.in}

\icmlkeywords{Deep Learning, Optimization, Gradient Descent, Stochastic Gradient Descent, Noise}
\vskip 0.3in
]

\printAffiliationsAndNotice{}

\begin{abstract}
Theoretical analysis of the error landscape of deep neural networks has garnered significant interest in recent years. In this work, we theoretically study the importance of noise in the trajectories of gradient descent towards optimal solutions in multi-layer neural networks. We show that adding noise (in different ways) to a neural network while training increases the rank of the product of weight matrices of a multi-layer linear neural network. We thus study how adding noise can assist reaching a global optimum when the product matrix is full-rank (under certain conditions established by \cite{DBLP:journals/corr/YunSJ17}). We establish theoretical foundations between the noise induced into the neural network - either to the gradient, to the architecture, or to the input/output to a neural network - and the rank of product of weight matrices. We corroborate our theoretical findings with empirical results.
\end{abstract}

\section{Introduction}
\label{sec1}
Deep neural network models are able to achieve state-of-the-art results on many real-world problems such as face recognition, speech recognition, and sentiment analysis. The recent empirical success of deep neural network models has evinced attempts by researchers to more closely understand the error landscape of deep neural networks, and analyze how a non-convex setting could consistently result in solutions of high application value. 

The line of research that has had a good impact in understanding the landscape of deep learning is by Choromanska \textit{et al.}\cite{pmlr-v40-Choromanska15}, where the authors study the error surfaces of deep neural nets under seven assumptions and analyzed them using the Hamiltonian of the spherical spin-glass model. This work leaves several open problems, an important one of which is to see if it is possible to drop some of the proposed assumptions and extend the error landscape analysis of neural networks. Kawaguchi addressed the proposed open problem in a seminal work \cite{kawaguchi2016deep} , where it was proved that the loss surface of deep neural nets are non-convex and non-concave; that all local minima are global minima; and that all other critical points are saddle points. This work was more recently extended by \cite{DBLP:journals/corr/YunSJ17} where the authors presented the necessary and sufficient conditions for the critical points to be global minima for deep linear networks. It was proposed that under a few assumptions, the rank of product of weight matrices at a global optimum is full-rank. We focus this effort on studying noise in neural networks, especially how the induced noise can help increase the rank of product of weight matrices. We show that all methods that involve noise, be it in the gradient during training, architecture of network, or added to input/output, attempt to increase the rank of the product of weight matrices as the optimization task progresses towards reaching a global minimum. The analysis of a linear network may look trivial at first sight, but even its loss function is non-convex in nature and, only recently have theoretical results started emerging for such networks. To the best of our knowledge, this is the first such effort where a unifying explanation of all such methods involving noise in neural networks are provided, and a connection of training methods to the rank of weight matrices and global optimality is studied. 

The remainder of this paper is organized as follows. Sec \ref{sec2} presents the notations and summarizes the key contributions of the paper; Sec \ref{sec3} establishes the relationship of noisy/perturbed gradient descent to optimality; Sec \ref{sec4} presents the connection of stochastic gradient descent to global optimality; Sec \ref{sec5} discusses how noise in architecture, input and output can be viewed in the same way; Sec \ref{sec6} presents a few extensions of our results (including to deep neural networks); Sec \ref{sec7} validates our results using experimental results; and Sec \ref{sec8} presents our conclusions.
\section{Preliminaries and Contributions}
\label{sec2}
In this section, we summarize the notations used as well as the key contributions of our work.
\vspace{-6mm}
\paragraph{Preliminaries:}
We consider a linear neural network with $H-1$ hidden layers each of which have a width $d_1,\cdots, d_{H - 1}$ respectively. The size of input and output layers are $d_x ,d_y$ respectively. Note that $d_x = d_0, d_y = d_H$. The hidden layer weights between layer $i-1$ and $i$ are given as $ W_i \in \mathbb{R}^{d_i \times d_{i-1}}$ for $i=1,\cdots, H$. The training data to the network are the input-output matrices $(X,Y)$, where $X \in \mathbb{R}^{d_x \times m}$, $Y \in \mathbb{R}^{d_y \times m}$ and $m$ is the total number of training samples. The loss function under consideration, as in \cite{kawaguchi2016deep,DBLP:journals/corr/YunSJ17}, is the squared loss error: $\mathcal{L}(W)= \frac{1}{2}|| W_{H}W_{H-1} \ldots W_1 X - Y||_F^2$. We define a ball centered at $c$ with radius $r$ as $\mathbb{B}_{F}(c,r)$, where $F$ is the Frobenius norm. 
\vspace{-6mm}
\paragraph{Our Key Contributions:} In this work, we theoretically analyze the influence of noise during training of neural networks. In particular, we show that \textit{perturbed gradient descent}, which adds noise to the gradient while training, increases the rank of the product matrix $\mathfrak{R} = W_{H}W_{H-1}\ldots W_1$. We then extend this analysis for other settings where noise is involved for neural networks. The key contributions of our work can be summarized as follows: we show that (i) for linear neural networks with \(H=2\), \textit{perturbed gradient descent} follows a trajectory that maintains a non-decreasing rank for $\mathfrak{R}$; (ii) the same results hold while training the network using Stochastic Gradient Descent (SGD); (iii) noise induced into the architecture as well as input/output also leads to a non-decreasing rank trajectory on $\mathfrak{R}$. We empirically validate our theoretical results by showing that using \textit{perturbed gradient descent} gradually increases the rank of the product of weight matrices, eventually reaching full rank and provide an extension to deep linear neural networks of our result under certain conditions, as well as empirically show that our result holds for deep networks.

\section{Noise in Batch Gradient Descent}
\label{sec3}
In this section, we show that for any two-layer linear neural network, training using perturbed gradient descent will increase the rank (rather, not decrease the rank)\footnote{For convenience, we use increasing and non-decreasing interchangeably in this work. We also use noise and perturbation interchangeably in this work.} of $\mathfrak{R}$ (product of weight matrices) in each iteration. Algorithm \ref{alg1} proposed by \cite{pmlr-v70-jin17a} is called \textit{perturbed gradient descent} as it adds a small amount of noise at every iteration to the calculated gradient.

\begin{algorithm}[H]
   \caption{Perturbed Gradient Descent \cite{pmlr-v70-jin17a}}
   \label{alg1}
\begin{algorithmic}
   \REQUIRE Initial weights $W^{0}_{i}, i = \{1, \ldots, H\}$
   \REQUIRE Learning rate $\eta$
   \REQUIRE Loss Function \(\mathcal{L} : X \times Y \to \mathbb{R}\)
   \STATE $t \leftarrow 0$ \hfill $\triangleright$ {\small Initialize time step}
   \REPEAT
   \FOR { $i \in \{1, \ldots, H \}$} 
   \STATE $\Delta W_i^{t} = \frac{\partial \mathcal{L}}{\partial W_i^{t}}$ \hfill $\triangleright$ {\small Gradient calculation}
    \STATE $W_i^{t+1} = W_i^t - \eta (\Delta W_i^{t} + \epsilon)$, \quad $\epsilon \sim \mathcal{N}(0,\sigma^2)$
   \ENDFOR
   \STATE \(t \leftarrow t + 1\)
   \UNTIL{convergence}
\end{algorithmic}
\end{algorithm}
\vspace{-10pt}

\begin{theorem}
\label{th:main}
Consider a $H$-layer linear neural network, trained using Algorithm \ref{alg1}, then $rank(\mathfrak{R}^t) \leq rank(\mathfrak{R}^{t+1})$, where $t$ is the current iterate, $\mathfrak{R}^t = W_H^t W_{H - 1}^t\ldots W_2^t W_1^t$ and $H=2$.
\end{theorem}
\paragraph{Proof Sketch:}
We use Lemmas \ref{lemma1} and \ref{rank-lemma} as the key steps towards our proof. We first show that we can increase the rank of a matrix by adding a small perturbation to the matrix in Lemma \ref{lemma1}. In Lemma \ref{rank-lemma}, we show that under certain conditions, if the rank of the individual weight matrices increase, then the rank of the product of weight matrices is also non-decreasing.

\begin{lemma}
\label{lemma1}
Given a matrix $A \in \mathbb{R}^{n \times m}$ with $\text{rank}(A)=r < \min\{ n,m \}$ and $A = U \Sigma V^T$ (Singular Value Decomposition of A), the matrix $\hat{A} = U \Sigma_{r+1} V^T$, where $\Sigma_{r+1} = \text{diag} \{ \sigma_1,  \sigma_2, \cdots,  \sigma_r,\epsilon,0,\cdots,0 \}$ has rank $r+1$ for all $\epsilon > 0$, and $||A-\hat{A}||_2 = \epsilon$.
\end{lemma}

Lemma \ref{lemma1} shows that matrix $A$ can be approximated by a high-rank matrix $\hat{A}$, by making perturbation to the $r+1^{th}$ entry in the $\Sigma$ matrix of the SVD decomposition of $A$, such that $||A-\hat{A}||_2 = \epsilon$. The proof of Lemma \ref{lemma1} and all subsequent lemmata are deferred to Appendix \ref{app-exp}.

\begin{lemma}
\label{angle-lemma}
In Lemma \ref{lemma1}, $\cos(A, \hat{A}) > 0$.
\end{lemma}

In other words, the high-rank approximation of matrix $A$ obtained using Lemma \ref{lemma1} makes an acute angle with the original matrix, thus allowing us to use this approximation for gradient descent while training.

We now proceed to show that the rank of the product of the rank-increased weight matrices also increases under certain conditions. Let $r_{Z}$ denote the rank of a matrix $Z$ for convenience.

\begin{lemma}
\label{rank-lemma}
Consider two matrices \(A\in \mathbb{R}^{m \times n}\), \(B \in \mathbb{R}^{n \times p}\), and a third matrix \(\hat{B}\in \mathbb{R}^{n \times p}\) such that \(r_{\hat{B}} = r_{B} + k\), $k \geq 0$. Then, given $r_A \geq n-k$, $r_B = n-k$, \(r_{A\hat{B}} \geq r_{AB}\).
\end{lemma}

\begin{corollary}
\label{cor-rank-1}
Consider two matrices \(A\in \mathbb{R}^{m \times n}\), \(B \in \mathbb{R}^{n \times p}\), and a third matrix \(\hat{A}\in \mathbb{R}^{m \times n}\) such that  \(r_{\hat{A}} = r_{A} + k\), $k \geq 0$. Then, given $r_A = n-k$, \(r_{\hat{A}B} \geq r_{AB}\).
\end{corollary}

\begin{corollary}
\label{cor-rank-2}
Consider two matrices \(A\in \mathbb{R}^{m \times n}\), \(B \in \mathbb{R}^{n \times p}\), and two more matrices \(\hat{A}\in \mathbb{R}^{m \times n}\),\(\hat{B} \in \mathbb{R}^{n \times p}\) such that  \(r_{\hat{A}} = r_{A} + k\), \(r_{\hat{B}} = r_{B} + k\), and $k \geq 0$. Then, given $r_A + r_B = n-2k$, \(r_{\hat{A}\hat{B}} \geq r_{AB}\).
\end{corollary}

We now see how increasing the rank of the product of weight matrices can help the training algorithm reach a global optimum. In recent work \cite{DBLP:journals/corr/YunSJ17}, Yun \textit{et al.} related the rank of product of weight matrices and global optimality and presented Theorem \ref{th1} (below), which gave the necessary and sufficient conditions of a critical point of a deep linear network to be a global minimum. The set of global minima is provided by partitioning the set of weight matrices based on the rank of the product of weight matrices of a neural network. The result holds true only under the following set of assumptions: (i) $\min \{d_{x}, d_{1}, d_{2}, \cdots, d_{H-1}, d_{y}\} = \min \{d_{x}, d_{y}\}$; (ii) $d_{x}, d_{y} \leq m$; (iii) $XX^{T}$ and $YX^{T}$ are full rank; (iv) singular values of $YX^{T}(XX^{T})^{-1}X$ are distinct.
\begin{theorem}
{\normalfont \citep[Thm 2.1]{DBLP:journals/corr/YunSJ17}}\hspace{2mm}
\label{th1}
If $k=\min\{ d_x, d_y\}$, define the following set
$\mathcal{V} = \{ (W_1,W_2,\cdots,W_{H}): rank(W_{H} \cdots W_2 W_1)=k\}$. Then every critical point in $\mathcal{V}$ is a global minimum and every other critical point in $\mathcal{V}^c$ is a saddle point.
\end{theorem}

\begin{theorem}
Under the conditions specified in Theorem \ref{th1}, perturbed gradient descent with $\text{rank}(\mathfrak{R}^t) \leq \text{rank}(\mathfrak{R}^{t+1})$ where $t$ is the current iterate is guaranteed to converge to a global minimum. 
\end{theorem}

The above theorem states that a critical point is a global minimum if the rank of the product of weight matrices is full rank. Hence, \textit{perturbed gradient descent} (Alg \ref{alg1}), while increasing the rank using induced noise, is guaranteed to reach a global minimum when the product matrix $\mathfrak{R}$ reaches full rank under the abovementioned set of assumptions. We note that the convergence analysis of \textit{perturbed gradient descent} is presented in \cite{pmlr-v70-jin17a} and does not affect our analysis.

\section{Noise via Stochastic Gradient Descent}
\label{sec4}
This section shows the equivalence between Stochastic Gradient Descent (SGD) and \textit{perturbed gradient descent}. By establishing the said equivalence, we hypothesize that SGD can also be viewed as increasing the rank of the product matrix, $\mathfrak{R}$, in each iteration. We assume the bounded variance property between batch and stochastic gradients in Property \ref{bvp}, as in  \citep[A1]{allen2017natasha}.
\begin{property}
\label{bvp}
(Bounded Variance Property) Given the full batch gradient $G$ and the stochastic gradient $g$ (of a mini-batch), $\mathbb{E}[|| g -  G ||]_2^2 \leq \gamma $ for some $\gamma > 0$, where \(||.||_{2}\) denotes the \(L_{2}\)-norm, and the expectation is taken over the mini-batches.
\end{property}

\begin{lemma}
\label{equi-lemma}
Under the assumption of property \ref{bvp}, with probability at least \(1 - \delta\), the following holds:
\begin{equation}
    \|\hat{g} - g\|_{2}    \leq {\Large \frac{(\sqrt{d\sigma^{2} + \gamma})}{\delta}}
\end{equation}
where \(g\) is the stochastic gradient (for a mini-batch) and \(\hat{g}\) is the perturbed full-batch gradient, given by $\hat{g} = G + \mathcal{N}(0, \sigma^2)$ (i.e. noise is sampled from zero-centered Gaussian with finite variance) and \(d\) is the dimension of $g$.
\end{lemma}

From the above result, we have shown an equivalence between the stochastic gradient and full batch gradient with noise \textit{i.e.,} perturbed gradient, and thereby, the connection between SGD and global optimality as in Section \ref{sec3}.

\section{Noise in Architecture}
\label{sec5} 

In this section, we analyze the usefulness of adding noise in different ways to the neural network architecture. We once again show that noise in architecture essentially helps the optimization algorithm increase the rank of the weight matrices. In particular, we study the addition of noise to the input/output, as well as a popular method: Dropout \cite{Srivastava:2014:DSW:2627435.2670313}.
\subsection{Effect of Noise in Input/Output}
We first show that adding noise to input helps increase the rank of the product matrix $\mathfrak{R}$, by establishing an equivalence with perturbed gradient descent. Let us define \( \mathfrak{R}_\ell^i = W_{i+1}^T \cdots W_H^T \) and \( \mathfrak{R}_r^i = W_1^T \cdots W_{i-1}^T \); we already know \( \mathfrak{R} = W_H W_{H-1} \cdots W_1\).

\begin{lemma}
\label{noiselemma}
A deep linear network trained with noise added to input, $X+\epsilon$, is equivalent to training with \textit{perturbed gradient descent} (Algorithm 1) where noise is a function of the weights and $\epsilon$ is given by $\varphi \left(W_{H}, W_{H-1}, \ldots, W_{1}, X;\epsilon \right) =  \mathfrak{R}_\ell^i [ \mathfrak{R} \epsilon X^T + \mathfrak{R} X \epsilon^T + R \epsilon \epsilon^T - Y\epsilon^T] \mathfrak{R}_r^i$.
\end{lemma}
It is easy to see that  when noise is added to the output as $Y+\epsilon$ in Lemma \ref{noiselemma}, the gradient of the loss function is a \textit{perturbed gradient} again. 

\subsection{Effect of Dropout:} Dropout \cite{Srivastava:2014:DSW:2627435.2670313} is a technique which injects multiplicative noise in the activation of a neural network:
\begin{equation}
\label{dropout}
O_k = W_k(Z \odot B)
\end{equation}
where $O_k$ is the output after Dropout, $W_k$ are the weights in the layer, $Z$ is the input before Dropout, and $B$ follows the Bernoulli distribution $B \sim \operatorname{Bern} \left({1-p}\right)$, where $p$ is the probability of success. Changing the underlying distribution of Dropout to Gaussian has been suggested in \cite{Srivastava:2014:DSW:2627435.2670313} and has been shown to work well in practice. In Lemma \ref{dropout-lemma} below, we show that applying Gaussian Dropout is equivalent to adding noise to the input.

\begin{lemma}
\label{dropout-lemma}
Let $\mathbb{E}[\epsilon \epsilon^T] =  \beta \mathbb{I}$, where $\mathbb{I}$ is the identity matrix, and $\beta > 0$. Then, the loss function of a linear neural network with Gaussian Dropout $G \sim \mathcal{N}(1,\sigma^{2})$ is the same as a network with loss function where $\epsilon$ is added to the input.
\end{lemma}

The aforementioned results show that noise introduced in the architecture, be it input/output or through (Gaussian) Dropout, is equivalent to \textit{perturbed gradient descent}, and thus increases the rank of the product of weight matrices, eventually helping reach one of the global minima.
\section{Extensions}
\label{sec6}
We now discuss possible extensions of the results presented in the work so far. In particular, we discuss the connection between \textit{perturbed gradient descent} and escaping saddle points during training.

\vspace{-5pt}
\paragraph{Escaping Saddle Points \textit{via} Rank Increase:} 
It is well-known that saddle points pose a significant problem \cite{Dauphin:2014:IAS:2969033.2969154} while training deep learning models. We present an alternate view for escaping saddle points through the proposed rank increase strategy. Let the highest attainable rank of $\mathfrak{R}$ (the product of weight matrices) be $r$. Lemma \ref{th:saddle} (informally) states our result.
\begin{lemma}
\label{th:saddle}
Perturbed gradient descent (Alg \ref{alg1}) escapes saddle points by increasing the rank of \( \mathfrak{R}\) under conditions specified in \citet[Thm~2.1]{DBLP:journals/corr/YunSJ17}.
\end{lemma}
The proof sketch for the above is straightforward from other results in this work. Assuming without loss of generality that \textit{perturbed gradient descent} is currently at a saddle point with $\text{rank}(\mathfrak{R}) = r-1$ (using necessary and sufficiency conditions stated in Thm \ref{th1}). Using Algorithm \ref{alg1}, we can increase the rank of $\mathfrak{R}$ by adding a small perturbation. When this is done sequentially and $\text{rank}(\mathfrak{R})$ is increased to the highest attainable rank $r$, by Theorem \ref{th1}, we know that we have reached a global minimum, thus escaping the saddle. 

\section{Experiments}
\label{sec7}
We conducted experiments to validate the claims in this paper, and the results (on linear, non-linear and deep networks) are presented in this section under the assumptions in Section \ref{sec3}. A synthetic dataset that ensures $XX^T$ and $XY^T$ are full rank (following the assumptions in \cite{DBLP:journals/corr/YunSJ17}) was used, with  input $X$ of 1000 dimensions, output $Y$ of 250 dimensions, and a total of $1000$ data points for training. The initial network architecture used was \texttt{1000 $\times$ 500 $\times$ 250}, and the results are shown in Figure \ref{rank-figure} for 50 epochs / full-batch iterations of training. We observed the rank of the product of weight matrices during training, while using gradient descent and \textit{perturbed gradient descent}. It can be clearly seen that when \textit{perturbed gradient descent} is used the product matrix attains a full rank of $250$, empirically validating our Theorem \ref{th:main}. This is, however, not the case with standard gradient descent (i.e., without noise). More results on other non-linear activation functions are deferred to Appendix \ref{app-exp}. 
\begin{figure}[H]
  \centering
\includegraphics[width=6cm, height=3cm]{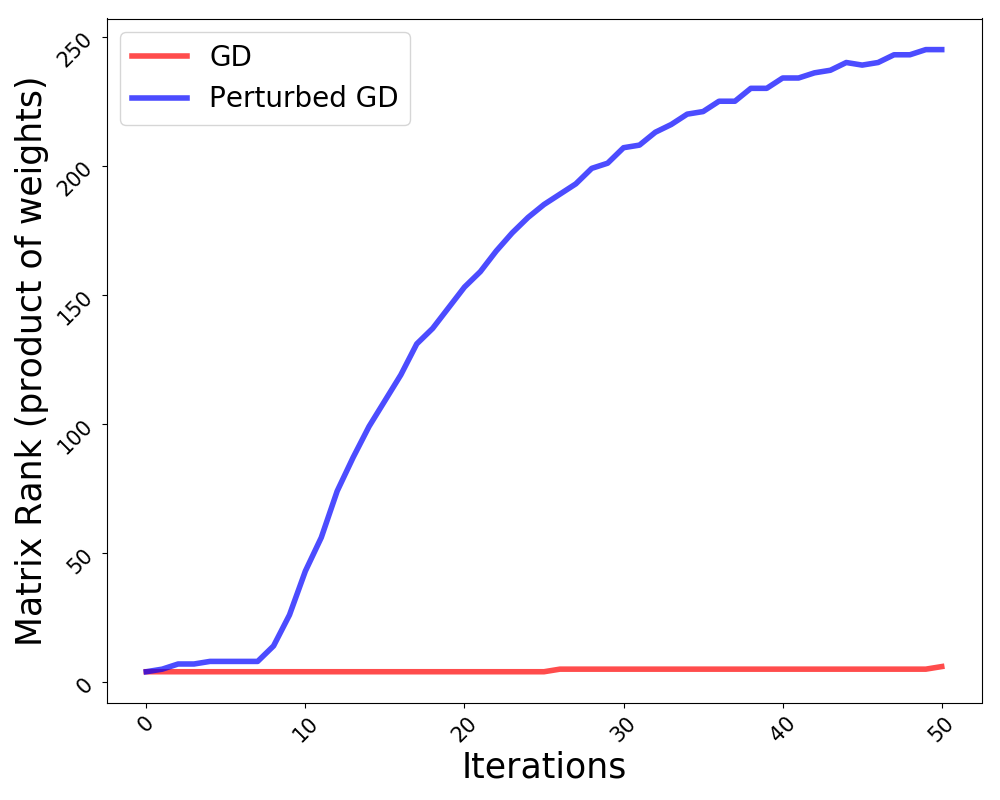}
  \caption{Rank of the product of weight matrices for the aforementioned architecture. Note that the rank of the product reaches 250, which is the highest possible in this scenario (where the dimensions of the matrix are \(250 \times 1000\)}
\label{rank-figure}
\end{figure}

\section{Conclusions}
\label{sec8}
In this work, we studied the importance of noise in the trajectories of training algorithm in linear neural networks. We analyzed noise in neural networks from different perspectives: gradient descent (including stochastic gradient descent), as well as architecture (including input/output, dropout). In all cases, we showed that noise helps increase the rank of the product of weight matrices in neural networks. We empirically evaluated our results on linear, non-linear and deep networks. We also discussed how the results in this work can be extended to deep networks under certain conditions and how under certain conditions, can ensure convergence to a global optimum. As future work, we plan to study the results while relaxing the assumptions, as well as extend the analysis to non-linear networks.

\bibliography{main}
\bibliographystyle{icml2018}
\newpage
\appendix
\onecolumn

\section{Appendix}
\addcontentsline{toc}{section}{Appendices}
\label{appendix}

\paragraph{Proof of Lemma \ref{lemma1}:}
\begin{lemma*}
\label{lemma1-app}
Given a matrix $A \in \mathbb{R}^{n \times m}$ with $\text{rank}(A)=r < \min\{ n,m \}$ and $A = U \Sigma V^T$ (Singular Value Decomposition of A), the matrix $\hat{A} = U \Sigma_{r+1} V^T$, where $\Sigma_{r+1} = \text{diag} \{ \sigma_1,  \sigma_2, \cdots,  \sigma_r,\epsilon,0,\cdots,0 \}$ has rank $r+1$ for all $\epsilon > 0$, and $||A-\hat{A}||_2 = \epsilon$.
\end{lemma*}
\begin{proof}
Let the SVD decomposition of $A$ be $U \Sigma_r V^T$. Let us define the high-rank approximation of matrix $A$ as $\hat{A}$ = $U \Sigma_{r+1} V^T$, where $\Sigma_{r+1} = \text{diag} \{ \sigma_1,  \sigma_2, \cdots,  \sigma_r,\epsilon,0,\cdots,0 \}$ with rank $r+1$. It can also be seen that \( \hat{A}\) is $\epsilon$-close to the original matrix \( A\). Bounding $|| A - \hat{A}||_2$ gives:
$|| A - \hat{A}||_2 = ||U (\Sigma_r - \Sigma_{r+1}) V^T||_2    =\sigma_{r+1} = \epsilon$.
\end{proof}
Lemma \ref{lemma1} shows that matrix $A$ can be approximated by a high-rank matrix $\hat{A}$, by making perturbation to the $r+1^{th}$ entry in the $\Sigma$ matrix of the SVD decomposition of $A$, such that $||A-\hat{A}||_2 = \epsilon$.

\paragraph{Proof of Lemma \ref{angle-lemma}:}
\begin{lemma*}
\label{angle-lemma-app}
In Lemma \ref{lemma1}, $\cos(A, \hat{A}) > 0$.
\end{lemma*}
\begin{proof}
Continuing with the same $A$ and $\hat{A}$ in Lemma \ref{lemma1}, consider the cosine of the angle between $A$ and $\hat{A}$: 
\begin{equation}
   \label{cos-mat-1}
   \cos\left(A, \hat{A}\right) = \frac{\mathrm{trace}(A^{T}\hat{A})}{||A||_{F}||\hat{A}||_{F}}
\end{equation}
Using the SVD of \(A\) and \(\hat{A}\):
\begin{gather*}
    \mathrm{trace}(A^{T}\hat{A}) = \mathrm{trace}(\Sigma_{A}^{T}\Sigma_{\hat{A}}) = \sum_{i=1}^r \sigma_{i}^{2} ;  \quad \text{Also} \quad
    ||A||_{F}^{2} = \sum_{i=1}^{r} \sigma_{i}^{2}; \qquad ||\hat{A}||_{F}^{2} = \sum_{i=1}^{r} \sigma_{i}^{2} + \epsilon^2
\end{gather*}
Using the above, we can show that:
\begin{equation}
\label{eq-angle-lemma}
    \cos(A, \hat{A}) = \frac{\sum_{i=1}^{r} \sigma^{2}_{i}}{\sqrt{\sum_{i=1}^{r} \sigma_{i}^{2}}\sqrt{\sum_{i=1}^{r} \sigma_{i}^{2} + \epsilon^2}} \\= \frac{\sqrt{\sum_{i=1}^{r} \sigma_{i}^{2}}}{\sqrt{\sum_{i=1}^{r} \sigma_{i}^{2} + \epsilon^2}} > 0
\end{equation}
\end{proof}

\paragraph{Proof of Lemma \ref{rank-lemma}:}
\begin{lemma*}
\label{rank-lemma-app}
Consider two matrices \(A\in \mathbb{R}^{m \times n}\), \(B \in \mathbb{R}^{n \times p}\), and a third matrix \(\hat{B}\in \mathbb{R}^{n \times p}\) such that \(r_{\hat{B}} = r_{B} + k\), $k \geq 0$. Then, given $r_B = n-k$, \(r_{A\hat{B}} \geq r_{AB}\).
\end{lemma*}
\begin{proof}
Assume instead that \(r_{A\hat{B}} < r_{AB}\). Using Sylvester's rank inequality and product rank inequality, we get:
\begin{equation}
\label{main-1}
\begin{split}
r_{A} + r_{\hat{B}} - n \leq r_{A\hat{B}} < r_{AB} \leq \min\{r_{A}, r_{B}\}\\
\implies r_{A} + r_{B} - (n - k) < \min\{r_{A}, r_{B}\} \quad \text{since } r_{\hat{B}} = r_{B} + k
\end{split}
\end{equation}
Two cases now arise: (i) \(r_{A} > r_{B}\); and (ii) \(r_{A} \leq r_{B}\). Consider case (i) when \(r_{A} > r_{B}\). Eqn \ref{main-1} then becomes:
\begin{equation*}
r_{A} + r_{B} - (n - k) < r_{B} \implies r_{A} < n - k
\end{equation*}
This is a contradiction given $r_A \geq n-k$ in the claim, and thus, \(r_{A\hat{B}} \geq r_{AB}\). \\
Similarly, consider case (ii) when \(r_{A} \leq r_{B}\). Eqn \ref{main-1} then becomes: 
\begin{equation*}
r_{A} + r_{B} - (n - k) < r_{A} \implies r_{B} < n - k
\end{equation*}
This is once again a contradiction given $r_B = n-k$ in the claim, and thus, \(r_{A\hat{B}} \geq r_{AB}\).
\end{proof}

\paragraph{Proof of Lemma \ref{equi-lemma}:}
\begin{lemma*}
Under the assumption of property \ref{bvp}, with probability at least \(1 - \delta\), the following holds:
\begin{equation}
    \|\hat{g} - g\|_{2}    \leq {\Large \frac{(\sqrt{d\sigma^{2} + \gamma})}{\delta}}
\end{equation}
where \(g\) is the stochastic gradient (for a mini-batch) and \(\hat{g}\) is the perturbed full-batch gradient, given by $\hat{g} = G + \mathcal{N}(0, \sigma^2)$ (i.e. noise is sampled from zero-centered Gaussian with finite variance) and \(d\) is the dimension of $g$.
\end{lemma*}
\begin{proof}
As given in the statement:
\begin{equation}
    \hat{g} = G + \mathcal{N}(0, \sigma^2)
\end{equation}
 We begin by noting that \(\mathbb{E}[||\hat{g} - G ||_{2}^{2}] = d\sigma^{2}\), since the noise added to each dimension is independent and identical to the noise added to other dimensions. Note that:
\begin{equation}
\label{split-sum}
  ||\hat{g} - G ||_{2}^{2} = ||\hat{g} - g||_{2}^{2} + ||g - G||_{2}^{2} \\ + 2\langle(\hat{g} - g), (g- G )\rangle
\end{equation}
Using the linearity of the inner product and the fact that \(\hat{g} - g = -(g - G ) + Z\), where \(Z \sim \mathcal{N}(0, \sigma^2)\), it can be shown that:
\begin{equation}
    \label{dot-prod-1}
    \langle (\hat{g} - g), (g- G) \rangle = -||g - G||_{2}^{2} \\+ \langle Z, (g - G) \rangle
\end{equation}
Using Eqn \ref{dot-prod-1} in Eqn \ref{split-sum}, we get:
\begin{equation}
    ||\hat{g} - G||_{2}^{2} = ||\hat{g} - g||_{2}^{2} - ||g - G||_{2}^{2} \\+ 2 \langle Z, (g - G) \rangle
\end{equation}
This can be used in expectation due to linearity as:
\begin{equation}
   \mathbb{E}[||\hat{g} - G||_{2}^{2}] = \mathbb{E}[||\hat{g} - g||_{2}^{2}] - \mathbb{E}[||g - G||_{2}^{2}] \\+ 2 \mathbb{E}[\langle Z, (g - G) \rangle] \\= \mathbb{E}[||\hat{g} - g||_{2}^{2}] - \mathbb{E}[||g - G||_{2}^{2}]
\end{equation}
The second equality is due to the zero-centeredness of \(Z\). Now using Property \ref{bvp}, we get:
\begin{equation}
    \mathbb{E}[||\hat{g} - g||_{2}^{2}] = d\sigma^{2} + \mathbb{E}[||g - G||_{2}^{2}] \\ \leq d\sigma^{2} + \gamma
\end{equation}

By using the concavity of the square root, we get:
\begin{equation}
    \mathbb{E}[||\hat{g} - g||_{2}] \leq \sqrt{\mathbb{E}[||\hat{g} - g||_{2}^{2}]} = \sqrt{d\sigma^{2} + \gamma}
\end{equation}

Finally by Markov's inequality, we get:
\begin{equation}
    P\left(||\hat{g} - g||_{2} > \frac{\sqrt{d \sigma^{2} + \gamma}}{\delta}\right) \leq \delta \\
    \Rightarrow P\left(||\hat{g} - g||_{2} \leq \frac{\sqrt{d \sigma^{2} + \gamma}}{\delta}\right) \geq 1 - \delta
\end{equation}
which completes the proof.
\end{proof}

\paragraph{Proof of Lemma \ref{noiselemma}:}
Restating the definitions, \( \mathfrak{R}_\ell^i = W_{i+1}^T \cdots W_H^T \) and \( \mathfrak{R}_r^i = W_1^T \cdots W_{i-1}^T \) and \( \mathfrak{R} = W_H W_{H-1} \cdots W_1\).

\begin{lemma*}
A deep linear network trained with noise added to input, $X+\epsilon$, is equivalent to training with \textit{perturbed gradient descent} (Algorithm 1) where noise is a function of the weights and $\epsilon$ is given by $\varphi \left(W_{H}, W_{H-1}, \ldots, W_{1}, X;\epsilon \right) =  \mathfrak{R}_\ell^i [ \mathfrak{R} \epsilon X^T + \mathfrak{R} X \epsilon^T + R \epsilon \epsilon^T - Y\epsilon^T] \mathfrak{R}_r^i$.
\end{lemma*}

\begin{proof}
The closed form equation for the derivative of the loss function with respect to the weights of the linear network is given as:
\begin{equation}
\label{grad-app}
    \left.\frac{\partial{\mathrm{L}}}{\partial{W_i}}\right\rvert_{X} = W_{i+1}^T \cdots W_{H}^T(W_{H}W_{H-1}\cdots W_1X-Y) X^T W_{1}^T \cdots W_{i-1}^T
\end{equation}
$\forall i=1,\cdots, H$.
Then Eqn \ref{grad-app} becomes:
\begin{align*}
    \left.\frac{\partial{\mathrm{L}}}{\partial{W_i}}\right\rvert_{X} &= \mathfrak{R}_\ell^i (\mathfrak{R} X - Y ) X^T \mathfrak{R}_r^i \\
    &= \mathfrak{R}_\ell^i \mathfrak{R} XX^T \mathfrak{R}_r^i - \mathfrak{R}_\ell^i YX^T \mathfrak{R}_r^i
\end{align*}
When a small perturbation is added to the input given by  $\widetilde{X} = X+\epsilon$, the gradient w.r.t. \(W_i\) changes to:
\begin{gather*}
    \left.\frac{\partial{\mathrm{L}}}{\partial{W_i}}\right\rvert_{\widetilde{X}} = \mathfrak{R}_\ell^i \mathfrak{R} (X+\epsilon)(X+\epsilon)^T \mathfrak{R}_r^i - \mathfrak{R}_\ell^i Y(X+\epsilon)^T \mathfrak{R}_r^i \\
    = \mathfrak{R}_\ell^i \mathfrak{R}XX^T\mathfrak{R}_r^i + \mathfrak{R}_\ell^i \mathfrak{R} \epsilon X^T \mathfrak{R}_r^i + \mathfrak{R}_\ell^i \mathfrak{R} X \epsilon^T \mathfrak{R}_r^i + \mathfrak{R}_\ell^i \mathfrak{R} \epsilon \epsilon^T \mathfrak{R}_r^i - \mathfrak{R}_\ell^i YX \mathfrak{R}_r^i - \mathfrak{R}_\ell^i Y \epsilon^T \mathfrak{R}_r^i
\end{gather*}
Simplifying the left-hand-side will result in separate terms including the actual gradient \(\left.\frac{\partial{\mathrm{L}}}{\partial{W_i}}\right\rvert_{X}\) due to linearity. Thus $\left.\frac{\partial{\mathrm{L}}}{\partial{W_i}}\right\rvert_{\widetilde{X}} = \left.\frac{\partial{\mathrm{L}}}{\partial{W_i}}\right\rvert_{X} + \mathfrak{R}_\ell^i [ \mathfrak{R} \epsilon X^T + \mathfrak{R} X \epsilon^T + R \epsilon \epsilon^T - Y\epsilon^T] \mathfrak{R}_r^i$. This modified form of gradient is a perturbation to the actual gradient.
\end{proof}

\paragraph{Proof of Lemma \ref{dropout-lemma}:}
\begin{lemma*}
Let $\mathbb{E}[\epsilon \epsilon^T] =  \beta \mathbb{I}$, where $\mathbb{I}$ is the identity matrix, and $\beta > 0$. Then, the loss function of a linear neural network with Gaussian Dropout $G \sim \mathcal{N}(1,\sigma^{2})$ is the same as a network with loss function where $\epsilon$ is added to the input.
\end{lemma*}
\begin{proof}
Consider the squared loss function of a two layered linear network as:
\begin{equation} 
\label{loss}
    \| Y - W_2 W_1X \|_F^2 = \mathrm{tr}\lbrace( Y - W_2 W_1X)^T ( Y - W_2 W_1X)\rbrace
\end{equation} 
Denoting $\odot$ as element-wise multiplication, and applying Gaussian dropout at layer 2, Eqn \ref{loss} becomes:
\begin{equation}
    \| Y - (W_2 \odot G) W_1X \|_F^2 \\= \mathrm{tr}\lbrace(Y - (W_2 \odot G) W_1X)^T (Y - (W_2 \odot G) W_1X)\rbrace
\end{equation}
where $G \sim \mathcal{N}(1,\,\sigma^{2})$% G is matrix of Gaussian random variables \(\mathcal{N}(1, \sigma^{2})\)
. Defining $M = (W_2\odot G)^T (W_2\odot G)$, we get:
\begin{equation}
\mathbb{E}_{G}[\| Y - (W_2 \odot G) W_1X \|_F ^2] \\= \mathbb{E}_{G}[\mathrm{tr}\lbrace Y^T Y - 2Y^T(W_2 \odot G) W_1 X  + X^TW_1^TMW_1X\rbrace]
\end{equation}

Simplifying further and using the fact that \(\mathrm{tr}\lbrace\mathbb{E}_{G}[Y^T(W_2 \odot G)W_1 X]\rbrace = \mathrm{tr}\lbrace Y^{T}(W_2 W_1 X) \rbrace\):
\begin{equation}
\mathbb{E}_{G}[\| Y - (W_2 \odot G) W_1X \|_F ^2] \\= \mathrm{tr}\lbrace Y^T Y - 2Y^T(W_2 W_1 X)  + \mathbb{E}_{G}[X^TW_1^TMW_1X]\rbrace
\end{equation}

Consider \(\mathbb{E}_{G}[M]\), where \(M_{ij}\) is the \((i, j)^{th}\) entry of the matrix. 
\begin{equation}
M_{ij} = \sum_{k} g_{ki} g_{kj} M_{ki} M_{kj} \implies \mathbb{E}_{G}[M_{ij}] \\= \sum_{k} \mathbb{E}_{G}[g_{ki} g_{kj}] M_{ki} M_{kj}
\end{equation}
\begin{equation*}
\mathbb{E}[g_{ki} g_{kj}] = \begin{cases} \mathbb{E}[g_{ki}] \mathbb{E}[g_{kj}] = 1 & \text{if } i = j \\ \mathbb{E}[g^{2}_{ki}] = 1 + \sigma^{2} & \text{if } i \neq j \end{cases} \quad \implies \quad \mathbb{E}_{G}[M] = \left(\mathrm{diag}(\sigma^{2}) + \mathbbm{1}\right) \odot (W_{2}^{T}W_2)
\end{equation*}
where \(\mathbbm{1}\) denotes a matrix of 1s.

Now, it is easy to see that: 
\begin{align*} 
\mathrm{tr}\lbrace\mathbb{E}_{G} [X^{T}W_{1}^{T} M W_{1} X] \rbrace &=  \mathrm{tr}\lbrace X^{T}W_{1}^{T}W_{2}^{T}W_{2}W_{1}X\rbrace + \mathrm{tr}\lbrace X^{T}W_{1}^{T}\mathrm{diag}(\sigma^{2})W_{2}^{T}W_{2}W_{1}X \rbrace  \\ 
 &=  ||W_{2}W_{1}X||_{F}^{2} + \sigma^{2}||W_{2}W_{1}X||_{F}^{2}
\end{align*}

Hence,
\begin{equation}
\label{19}
\mathbb{E}_{G} [\| Y - (W_2 \odot G) W_1X \|_F ^2] = ||Y - W_{2} W_{1} X||_{F}^{2} + \sigma^{2} ||W_{2} W_{1} X||_{F}^{2}
\end{equation}

Let's now consider the same loss with noise $\epsilon \sim \mathcal{N}(0,\gamma)$ added to input:
\begin{equation}
\mathbb{E}_{\epsilon}[\| Y - W_2 W_1X(I+ \epsilon) \|_F^2] \\= \mathbb{E}_{\epsilon}[\| Y-W_2 W_1 X - W_2 W_1 X \epsilon\|_F^2]
\end{equation}

\begin{align} 
\mathbb{E}_{\epsilon}[\| Y - W_2 W_1X(I+ \epsilon) \|_F^2] &=  \mathbb{E}_{\epsilon}[||(Y - W_2 W_1 X)||_{F}^{2} + \mathbb{E}_{\epsilon}[||W_2 W_1 X \epsilon||_{F}^{2}] \\
 &  - 2\hspace{1mm}\mathrm{tr}\lbrace(Y - W_{2}W_{1}X)^{T}(W_{2}W_{1}X\epsilon)\rbrace]  
\end{align}

Since $\mathbb{E}[\epsilon] = 0$ and $\mathbb{E}[\mathrm{tr}(\cdot)] = \mathrm{tr}(\mathbb{E}[\cdot])$, we get:
\begin{align*} 
\mathbb{E}_{\epsilon}[\| Y - W_2 W_1X(I+ \epsilon) \|_F^2] &=  \mathbb{E}_{\epsilon}[||(Y - W_2 W_1 X)||_{F}^{2} + \mathbb{E}_{\epsilon}[||W_2 W_1 X \epsilon||_{F}^{2}] \\ 
 &=  ||Y - W_2 W_1 X||_{F}^{2} + \mathbb{E}_{\epsilon}[||W_2 W_1 X \epsilon||_{F}^{2}] \\
 &= \| Y - W_2W_1X \|_F^2 + \mathrm{tr}\{\mathbb{E}_{\epsilon}[\epsilon \epsilon^T X^T W_1^T W_2^T W_2W_1 X]\}
\end{align*}

Given $\mathbb{E}[\epsilon \epsilon^T] =  \beta \mathbb{I}$, we now have: 
\begin{align}
\label{21}
    E[\| Y - W_2 W_1 X(I+ \epsilon) \|_F^2] &= \| Y - W_2W_1X \|_F^2 + \beta \mathrm{tr}\{X^{T}W_{1}^{T}W_{2}^{T}W_{2}W_{1}X \}\\
    &= \| Y - W_2W_1X \|_F^2+ \beta \| W_2W_1 X\|_F^2
\end{align}
%\begin{multline}
%\label{21}
%    E[\| Y - W_2 W_1 X(I+ \epsilon) \|_F^2] = \| Y - W_2W_1X \|_F^2 + \beta \mathrm{tr}\{X^{T}W_{1}^{T}W_{2}^{T}W_{2}W_{1}X \}\\= \| Y - W_2W_1X \|_F^2+ \beta \| W_2W_1 X\|_F^2
%\end{multline}
Thus, Dropout (Eqn \ref{19}) is equivalent to adding an appropriate noise in the input (Eqn \ref{21}).
\end{proof}
(The assumption of \( \mathbb{E}[\epsilon \epsilon^T] =  \beta \mathbb{I} \) in Lemma \ref{dropout-lemma} is motivated from \citep[Main Thm 13]{pmlr-v40-Ge15}).

\paragraph{More Experimental Results:}
\label{app-exp}
This section presents an extensive set of experimental results validation our theoretical findings. Figure \ref{sigmoid-result} presents the results of rank increment of individual and product matrix on a deep non-linear network. We used a network of \texttt{1000 $\times$ 500 $\times$ 250} with a sigmoid activation at the hidden layers. It can be seen from Figure \ref{sigmoid-result} that the product matrix reaches a full rank.

\begin{figure}[H]
  \centering
  \includegraphics[width=6cm, height=3cm]{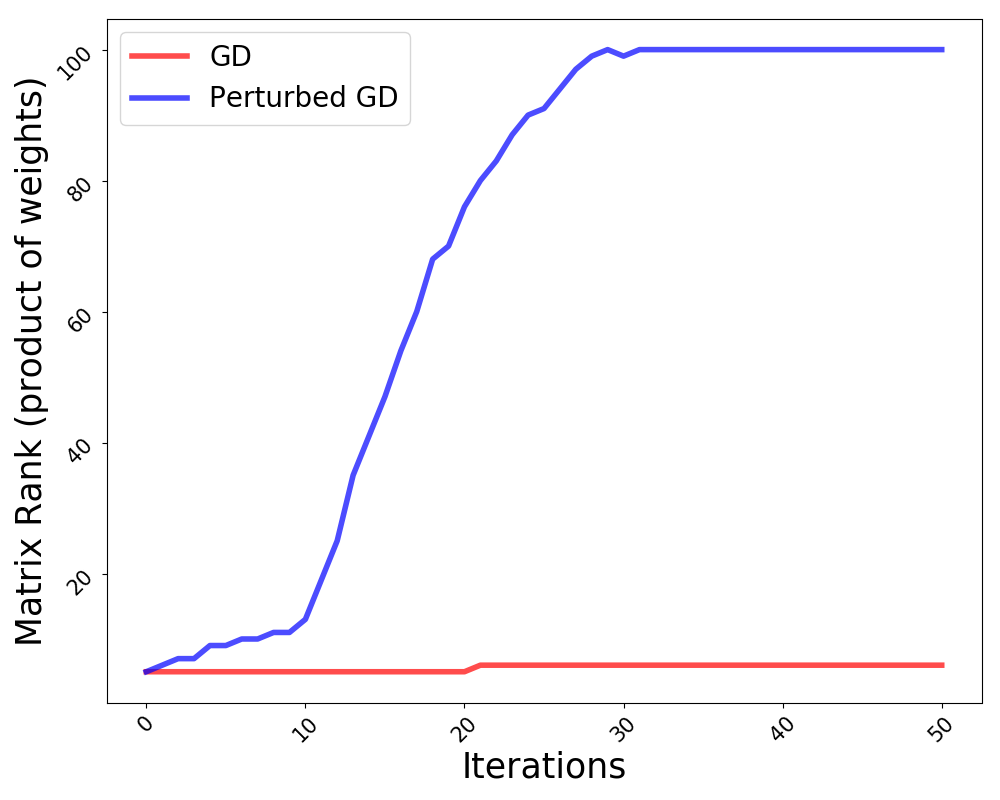}
  \caption{Rank of product of weight matrices for the aforementioned architecture. Note that the rank of the product reaches \(250\) which is the highest possible in this scenario (where the dimensions of the matrix are \(250 \times 1000\)}
\label{sigmoid-result}
\end{figure}

The same experiment is repeated with another non-linear activation function - \texttt{tanh} on a network architecture \texttt{900 $\times$ 500 $\times$ 100} with the non-linearity at the hidden and output layers, and the results are shown in Figure \ref{tanh-result}.

\begin{figure}[H]
  \centering
\includegraphics[width=6.5cm,height=4cm]{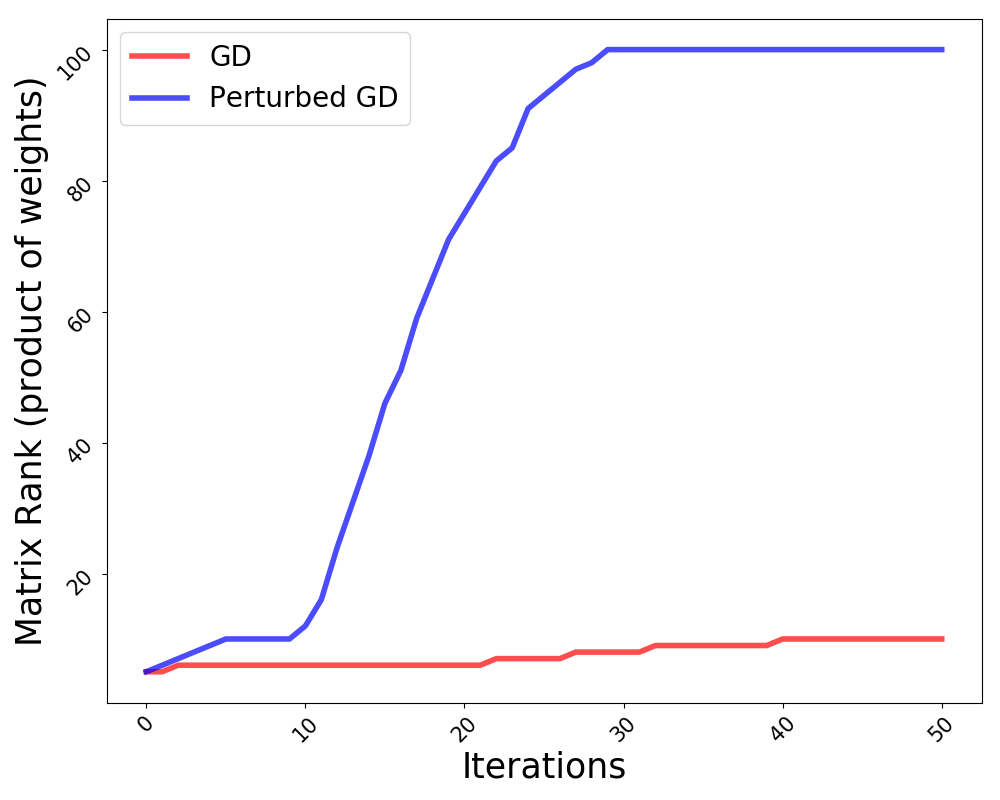}
  \caption{Rank of product of weight matrices for the aforementioned architecture. Note that the rank of the product reaches \(100\) which is the highest possible in this scenario (where the dimensions of the matrix are \(100 \times 900\)}
\label{tanh-result}
\end{figure}

The  result is also verified on linear networks with deeper architecture \((H=4,5)\) and the results are shown in Figure \ref{deep-result}. In all the experiments, it can be clearly seen that as the optimization task progresses, Alg \ref{alg1} results in the product of the weights approaching full rank. 
\begin{figure}[H]
  \centering
  \subfigure[Architecture:\newline{}\texttt{1000 x 700 x 500 x 200 x 100}]{\includegraphics[width=5cm, height=3cm]{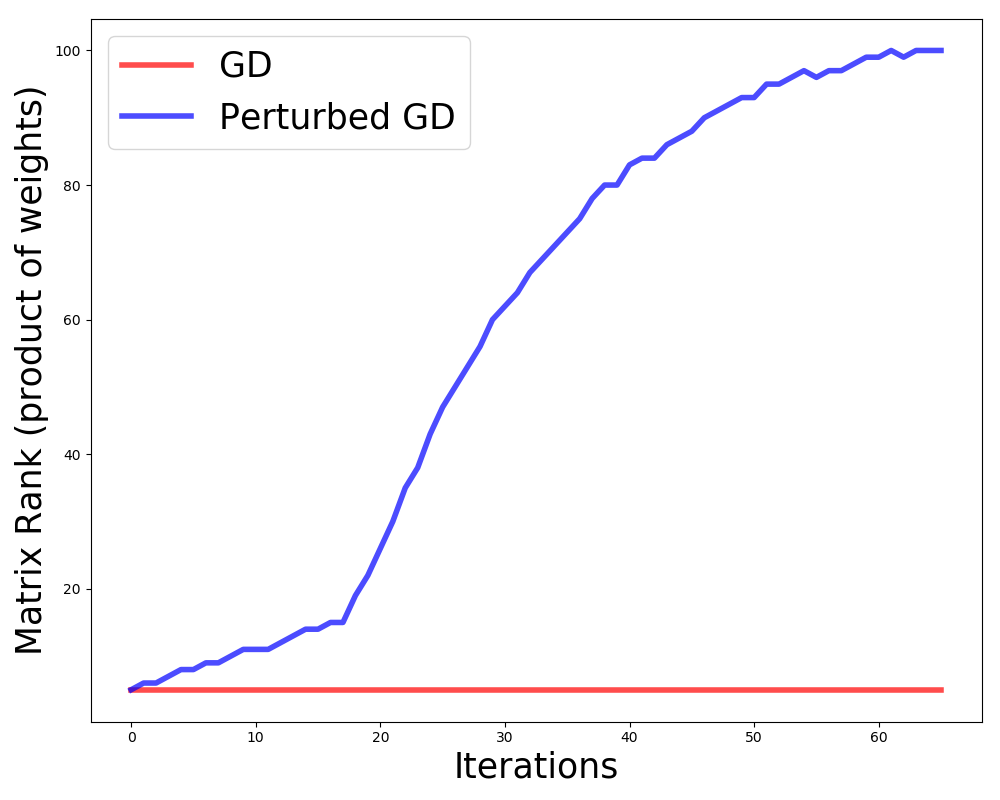}}\quad
  \subfigure[Architecture:\newline{}\texttt{1000 x 700 x 600 x 400 x 200 x 100}]{\includegraphics[width=5cm, height=3cm]{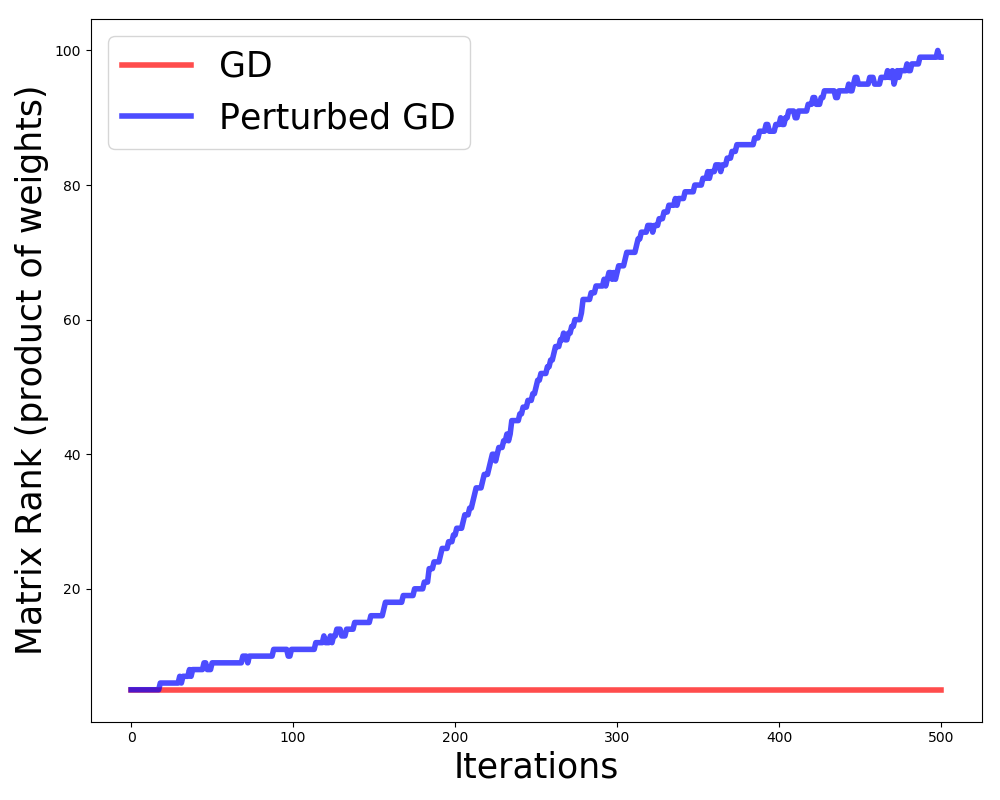}}
  \caption{Rank of product of weight matrices of deep linear networks}
\label{deep-result}
\end{figure}

\end{document}